\documentclass[10pt]{article}
\date{}

\usepackage{times}
\usepackage{soul}
\usepackage{url}
\usepackage[utf8]{inputenc}
\usepackage{graphicx}
\usepackage{amsmath}
\usepackage{booktabs}
\usepackage{algorithm}
\usepackage{algorithmic}
\usepackage{microtype}
\usepackage{subfigure}
\usepackage{float}
\usepackage{amsmath,amssymb,mathrsfs}
\usepackage{color}
\usepackage{hyperref}
\usepackage{epstopdf}
\usepackage{dsfont}
\usepackage{ifpdf}
\usepackage{multirow}

\ifpdf
    \graphicspath{{figures/PNG/}{figures/PDF/}{figures/}}
\else
    \graphicspath{{figures/EPS/}{figures/}}
\fi

\newtheorem{lemma}{Lemma}
\newtheorem{remark}{Remark}
\newtheorem{assumption}{Assumption}

\newcommand{\dff}{\stackrel{\scriptscriptstyle\triangle}{=}}

\DeclareMathOperator*{\argmax}{arg\,max}

\def \A{{\cal A}}

\def \E{{\cal E}}

\def \H{{\cal H}}

\def \bP{\mathbb{P}}
\def \bE{\mathbb{E}}
\def \bN{\mathbb{N}}

\begin{document}
\title{Reinforcement Learning-Based Coverage Path Planning with Implicit Cellular Decomposition }
%
%
\author{Javad Heydari \and Olimpiya Saha \and Viswanath Ganapathy \thanks{All authors are with the LG Electronics Advanced AI Team, Santa Clara, CA 95054.} \thanks{email: \tt{khormizi@gmail.com, osaha@unomaha.edu, gviswa@gmail.com.}}}

\maketitle              
\begin{abstract}
Coverage path planning in a generic known environment is shown to be NP-hard. When the environment is unknown, it becomes more challenging as the robot is required to rely on its online map information built during coverage for planning its path. A significant research effort focuses on designing heuristic or approximate algorithms that achieve reasonable performance. Such algorithms have sub-optimal performance in terms of covering the area or the cost of coverage, e.g., coverage time or energy consumption. In this paper, we provide a systematic analysis of the coverage problem and formulate it as an optimal stopping time problem, where the trade-off between coverage performance and its cost is explicitly accounted for. Next, we demonstrate that reinforcement learning (RL) techniques can be leveraged to solve the problem computationally. To this end, we provide some technical and practical considerations to facilitate the application of the RL algorithms and improve the efficiency of the solutions. Finally, through experiments in grid world environments and Gazebo simulator, we show that reinforcement learning-based algorithms efficiently cover realistic unknown indoor environments, and outperform the current state of the art.

\end{abstract}

\section{Introduction}

Autonomous robots are ubiquitous nowadays and expected to grow in popularity and deployment well into the future. In some application domains, such as vacuum cleaning, lawn mowing, and painting, it is of interest for an intelligent autonomous robot to be able to cover the entirety of a given environment~\cite{chen2013path,einecke2018boundary,ismail2015landmines,lehnert2018sweet,prabakaran2018floor}. Hence, the robot should be able to plan a path so as to visit the entire environment not occupied by an obstacle, while at the same time controlling some notions of cost such as revisiting already covered areas, i.e., overlap, coverage time, number of turns, or energy consumption~\cite{galceran2013survey,cabreira2019survey,Pragnavi2020}.

Coverage path planning (CPP) algorithms, based on their knowledge of the environment, in terms of the size and the shape of the environment, the number of obstacles, and the obstacle locations, can be categorized into offline and online methods. The offline CPP problem, where the robot has full geometric description of the area of interest, is shown to be NP-hard~\cite{Arkin2000}. Many approximation or heuristic algorithms have been proposed, such as the boustrophedon or Morse decomposition based coverage algorithms~\cite{choset1998coverage,Mannadiar2010,Morse2002,latombe2012robot}, the spiral path coverage~\cite{Gonzalez2005}, and the spanning-tree based coverage~\cite{Gabriely2001}. Further improvements in coverage have been explored using genetic algorithms~\cite{jimenez2007optimal} or computing the intrinsic coordinate system using holomorphic quadratic differentials for generic 3D environments~\cite{Lin2017}. Area coverage using multiple robots have also been studied in~\cite{adepegba2016multi,fazli2010multi}.

The full-knowledge assumption is not practical, especially for commercial robots that are manufactured for general purposes such as vacuum cleaning and lawn mowing and are deployed in a diverse set of environments. In the online version of the problem, the robot has no, or only partial, knowledge about the size and geometry of the area to be covered, or shapes and locations of the obstacles. In such scenarios, the robot accumulates the knowledge of the environment over time using on-board sensors and data storage, and builds an online map of the area. The boustrophedon coverage, spiral path, and spanning tree techniques are adapted to the online version of the problem through using the online map~\cite{Choi2009, BA*, Acar2006, Hert1999, Viet2015, Pham2019}. The online coverage with energy constraint for a battery-powered robot is studied in~\cite{Wei2018, Sharma2019, Shnaps2016, dutta2019constant}. In~\cite{hassan2019ppcpp, hassan2020dec}, both online and offline maps are used at the same time, where offline map shows the location of the static obstacles, while the online map  keeps track of the moving agents/obstacles in the environment.

The performance of all the aforementioned online algorithms hinges heavily on the accuracy of the area map generated by the robot while moving around the environment. Therefore, while these algorithms are easily implementable on robotic platforms, they suffer from missing some free regions and significant overlap. On the other hand, they are universal algorithms making them unable to adapt to the environments they are being used, while in many applications, such as vacuum cleaning and lawn mowing, the robot is used in the same set of environments throughout its lifetime. To address both issues, in this paper, we propose reinforcement learning to enable the robot to learn both the model of the environment including the sensor noise model, and to adapt to the set of environments the robot is being deployed in over time.

In this paper, by accounting for the uncertainty of 1) the unknown environment, 2) the sensor noise model, and 3) robot movement model, we formulate the coverage path planning in an unknown environment as a stochastic optimization problem in the framework of optimal stopping rule. By analyzing the problem in the stochastic domain, we formalize the dynamic programming solution to the problem, for which we identify the properties of the problem that makes it amenable for computationally solving the problem by using reinforcement learning (RL). The transition from stochastic optimization to RL facilitates some of the major difficulties of applying RL for general control problems such as choosing the reward function or state representation.

Application of RL to various control tasks has grown since its success in playing Atari games~\cite{mnih2013playing}. Recently, RL has been used for different path planning tasks such as navigation~\cite{koutnik2013evolving,zhao2015stacked,oh2016control,tessler2017deep,tai2016towards,mirowski2016learning,mirowski2018learning}, localization~\cite{chaplot2018active,zhu2017target}, and mapping~\cite{zhang2017neural,parisotto2017neural}. Rl has also been applied for coverage in different contexts. For instance, the studies in~\cite{Rudolph,Seah} apply RL to control sensor networks for covering a surveillance area, and~\cite{app9224964} uses RL to control multiple UAVs for searching an area. This is in contrast to the objective of this paper in which one robot is given the task of covering an environment by moving around while avoiding the obstacles.
	
The paper's contribution is threefold: first, we formulate the CPP as a stochastic optimization problem. Then, we demonstrate that RL is a natural computational approach for solving the optimization and the reward function and state representation come naturally from the optimization problem. Finally we validate the performance of an RL agent through simulations in various maze environments as well as the physics-based Gym-Gazebo simulator~\cite{zamora2016extending}.

\section{Model and Formulation}

\subsection{Model}

We consider a typical indoor environment $\E$, which consists of multiple obstacles and its size is at most $L\times W$. We assume that the environment, including its size, the number, locations, and the sizes of obstacles, is generated randomly according to a distribution denoted by $P_E$. A robot with diameter $d$ is tasked with covering environment $\E$ by moving around. We assume that the area of the environment $\E$ that is accessible by the robot is $A_E$ and note that $A_E$ is a random variable. When the robot tries to cover the environment, it revisits some area of $\E$ multiple times resulting in overlap and increasing the coverage time. The ultimate goal of the robot is to cover the accessible area of the environment in the quickest way, which is equivalent to minimizing the overlap. In this paper, we assume that the robot has no energy constraint, i.e., its battery capacity is sufficiently large to cover the whole area in one shot. Furthermore, the robot is not required to return to the starting point after finishing the coverage.

In order to formalize this problem, we first discretize the environment into cells with size equal to the size of the robot. Without loss of generality and for convenience in notations, we assume that $d=1$ and $L,W\in\mathbb{N}$. Furthermore, we assume that the robot interaction with the environment, such as sensing, moving, and turning, is performed in discrete time steps denoted by $t\in\mathbb{N}$\footnote{$t$ can be considered as small as required in order to make the robot movement natural.}. The robot has a set of possible actions $\A$ and it takes action $a_t\in\A$ at time $t$ to navigate around the environment. In a simple grid-world environment we may have $\A=\{\text{up, down, left, right}\}$, while for a real robot, due to the physical constraint of the robot movement, we usually have $\A=\{\text{move forward, rotate left, rotate right , \dots}\}$. As the robot moves around for coverage, it collects information about the environment via its sensors including bumper sensors, LiDAR, and 3D sensors. We denote the set of sensory data at time $t$ by $Y_t$ and the data accumulated up to time $t$ by $Y^t\dff(Y_1,Y_2,\dots, Y_t)$.

As the robot moves around the environment, it accumulates more information about the unknown environment and reduces its uncertainty about the area. We denote the filtration generated by the information accumulated up to time $t$ by the $\sigma$-algebra $\H_t$, i.e.,
\begin{align}
\H_t\dff\sigma(a_1,Y_1,\dots,a_t,Y_t)\ .
\end{align}
At each time step $t$, and based on the information accumulated up to that time, i.e. $\H_t$, the robot takes one of the following two possible actions:
\begin{enumerate}
\item {\em Stopping:} the robot is confident enough that it has achieved the desirable coverage performance, and stops the navigation process.
\item {\em Moving:} due to the lack of confidence in achieving the desired coverage, the robot decides to continue the coverage process by taking at least one more step. In this case, the robot must select action $a_{t+1}\in\A$.
\end{enumerate}
We define the stopping time $\tau$ as the time instant at which the robot is confident enough about achieving the desired coverage performance and takes the {\em stopping} action. We note that $\tau$ is a random variable as it is a function of random sequences of sensor measurements and the environment randomness. Prior to the stopping time, for time $t\in\{1,2,\dots,\tau-1\}$ we define the set of action selection functions $\psi^\tau\dff\{\psi_t\}_{t=1}^{\tau-1}$, where $\psi_t$ is a $\H_t$-measurable function that specifies action $a_t$. Finally, we define $\pi\dff(\tau,\psi^\tau)$ as the coverage path planning policy.

\subsection{Problem Formulation}

The optimal policy can be determined by specifying the optimal stopping time and the action selection functions. The ultimate goal in a coverage path planning problem is to cover the environment as quickly as possible. Therefore, reasonable coverage and quickness are the main figures of merit in this coverage path planning problem. For a random environment, the coverage performance can be quantified through the probability of achieving a certain level of coverage. Toward this end, we define $C$ as the random variable denoting the coverage percentage, and $P_c^\pi(\eta)$ as the probability of achieving coverage percentage $\eta$ by policy $\pi$ at the stopping time, i.e.,
\begin{align}
P_c^\pi(\eta)\dff\bP(C>\eta\;|\;\pi)\ ,
\end{align}
where $\bP$ is the probability measure that captures the randomness in the environment model as well as the robot movements and its sensing. The other figure of merit for the CPP problem is the quickness of the process, which is captured by the average number of steps required by the robot following  policy $\pi$ to achieve the desirable coverage level, i.e., $\bE\{\tau\,|\,\pi\}$\footnote{To simplify the notations, whenever it can be inferred from the context, the policy $\pi$ in the definition of coverage and quickness measures are dropped.}. There exists an inherent tension between the coverage performance and quickness of any policy as improving the coverage by moving around the environment increases the delay. In this paper, since we are interested in achieving a certain coverage performance, we control the coverage percentage and minimize the delay. As a result, the optimal policy would be the solution to the following optimization problem:
\begin{align}\label{eq:p1}
\begin{array}{ll}
\min_\pi & \bE\{\tau\,|\,\pi\} \\
\text{s.t.} & P_c^\pi(\eta)\geq\alpha
\end{array}\ ,
\end{align}
where $\alpha\in(0,1)$ controls the balance between the coverage and delay. The coverage problem in~\eqref{eq:p1} is an optimal stopping problem and it has been shown that for any $\alpha$ there exists a positive constant $\lambda_\alpha$, for which~\eqref{eq:p1} is equivalent to
\begin{align}
\max_\pi\ P_c^\pi(\eta)-\lambda_\alpha\cdot \bE\{\tau\,|\,\pi\} \ .
\end{align} 
Here, $\lambda_\alpha$ is a function of $\alpha$ and can be considered as the delay cost, i.e. each step has a cost of $\lambda_\alpha$. Note that the policy $\pi$, generally, is a stochastic policy, i.e. it returns a probability distribution over the set of actions $\A$. However, it has been shown that for a stationary environment, the optimal solution can be achieved by a deterministic policy as well. In the remainder of this paper, we aim to solve this optimization problem via reinforcement learning (RL).

\section{Finite-Horizon Coverage}

In the finite-horizon setting, the stopping time is upper bounded by  a pre-specified horizon $T\in\mathbb{N}$, i.e. $\tau\leq T$. In the coverage problem considered in this paper, the horizon can be thought of as the battery capacity of the robot which forces the robot to cover as much as possible before draining its battery. Finite-horizon analysis provides an insight to the general solution in the infinite-horizon setting. We define $\mu_t$ as the joint posterior distribution of the environment map, $M$, the robot position $P$, and the coverage $C$, at time $t$, i.e.,
\begin{align}
\mu_t\dff\bP(M,P,C\,|\,\H_t)\ ,\quad \forall t\in\{1,\dots,T\}\ .
\end{align}
Let us define the value function of policy $\pi$, also known as the expected reward-to-go, at time $t$ by $\tilde{V}_t^{T;\pi}(\H_t)$. Then, from the optimal stopping theorem we have
\begin{align}
\tilde{V}_t^{T;\pi}(\H_t)=\max&\,\big\{\bP(C>\eta\,|\,\H_t) \, ,\, -\lambda_\alpha+\bE_\pi\{\tilde{V}_{t+1}^{T;\pi}(\H_{t+1})\,|\,\H_t\}\big\}\ ,
\end{align}
where $\bE_\pi$ is the expectation over the policy $\pi$ and next observation $Y_{t+1}$, and it computes the expected coverage performance if one more step is taken considering all the available information $\H_t$. On the other hand, $\bP(C>\eta\,|\,\H_t)$ captures the coverage performance if the process stops at time $t$.
The lemma below asserts that the optimal stopping time and the action selection rules at time $t$ depend on $\H_t$ only through $\mu_t$. First, we need the following assumption.

\begin{assumption}
Given the map of the environment, $M$, and the position of the robot, $P$, the next observation $Y_{t+1}$ is independent of the past observations , i.e.,
\begin{align}
\bP(Y_{t+1}\,|\,\H_t,M,P)=\bP(Y_{t+1}\,|\,M,P)\ .
\end{align}
\end{assumption}
This is not a restricting assumption as when the sensor noises are independent across time the assumption holds.
\begin{lemma}[Sufficient Statistics]\label{lemma:ss}
The value function $\tilde{V}_t^{T;\pi}(\H_t)$ depends on $\H_t$ only through $\mu_t$, and can be denoted as a function of $\mu_t$ by $V_t^{T;\pi}(\mu_t)$.
\end{lemma}
\begin{proof}
First, we prove the recursive connection between $\mu_{t+1}$ and $\mu_t$ for $t\in\bN$:
\begin{align}
\mu_{t+1} &= \bP(M,P,C\,|\,\H_{t+1}) \\
&= \bP(M,P,C\,|\,\H_t,Y_{t+1}) \\
&= \frac{\mu_t\bP(Y_{t+1}\,|\,Y_t,M,P,C)}{\sum_{\mu'\sim\bP(\cdot\,|\,\H_{t})}\mu'\bP(Y_{t+1}\,|\,Y_t,M,P,C)}\ .
\end{align}
Next, we note that at time $T$ the covering process should stop and, consequently, we have
\begin{align}\label{eq:F}
\tilde V_T^{T;\pi}(\H_T)=\int_m\int_p\int_{c=\eta}^{100} \mu_T\dff F(\mu_T)\ .
\end{align}
Next, by using the backward induction and using the recursive property of sequence $\{\mu_t\}$ the proof is concluded, which we remove for the sake of brevity. 
\end{proof}

\section{Infinite-Horizon Coverage}

By assuming that there is no energy constraint, we set $T\rightarrow\infty$. Then, the limit $\lim_{T\rightarrow\infty} V_t^{T;\pi}(\mu_t)$ is well-defined according to monotonic convergence theorem as $V_t^{T;\pi}(\mu_t)$ is increasing in $T$ and is upper-bounded by $1$. By invoking Lemma~\ref{lemma:ss}, it can be readily shown that $V_t^{\infty;\pi}(\mu_t)$ is independent of time step $t$, i.e., $V_t^{\infty;\pi}(\mu)=V_{t'}^{\infty;\pi}(\mu)$ for $t\neq t'$. In order to highlight this observation, throughout the rest of this paper we use the shorthand $V^{\pi}(\mu)$ to denote this function. Thus, for the reward-to-go function in the infinite-horizon setting, we have
\begin{align}\label{eq:bellman1}
V^{\pi}(\mu_t)=\max\big\{F(\mu_t) \,,\, -\lambda_\alpha+ \bE_\pi\{V^{\pi}(\mu_{t+1})\,|\,\mu_t\}\big\} \ .
\end{align}
This is the Bellman equation, the optimal solution to which, i.e., $\pi^*=\argmax_\pi\ V^\pi(\mu)$, $\forall \mu$, is the optimal solution to the original CPP problem in~\eqref{eq:p1}. 
\begin{remark}
Function $F(\cdot)$, as defined in~\eqref{eq:F}, depends on the model of the environment which is unknown. Later in the paper, we provide some practical methods to compute $F$.
\end{remark}
The analytical solution to this problem requires the entire model of the environment and robot, which is assumed to be unknown in this paper. Hence, we use the computational techniques of the reinforcement learning literature to solve it. In the remainder of the paper, we discuss all the practical and technical considerations used to solve the CPP more efficiently.

\section{Reinforcement Learning for CPP}
\label{sec:RL}
	
Reinforcement learning (RL) handles the problem of an \textit{agent} (the robot) learning to act in an \textit{environment} (the area to be covered and its dynamics), with the goal of maximizing a predefined scalar \textit{reward} signal. Such learning problem is formulated as a Markov Decision Process (MDP). An MDP is a tuple $\langle \mathcal{S},\mathcal{A},P,R,\gamma\rangle$, where $\mathcal{S}$ and $\mathcal{A}$ are the sets of states and actions, respectively, and $P$ and $R$ are the corresponding state transition probabilities and the reward function, respectively. Constant $\gamma\in[0,1]$ is the discount factor which trades off between immediate and future reward. At each discrete time step $t$, the agent acquires an observation $s_t$ from the environment, selects a corresponding action $a_t$, then receives feedback from the environment in the form of a reward $r_t=R(s_t,a_t)$ and the updated state information $s_{t+1}$. The goal of RL agent is select policy $\pi$ to maximize the discounted sum of future rewards, i.e., $V^\pi(s_1)=\sum_{t=1}^{\tau}\gamma^t R(s_t,a_t)$, which according to the Bellman optimality principle satisfies
\begin{align}\label{eq:bellman2}
V^\pi(s_t) = \left\{\begin{array}{ll}
r_t+\gamma\cdot\bE\{V^\pi(s_{t+1})\} & t<\tau \\
r_t & t=\tau
\end{array} \right.\ .
\end{align}
Comparing~\eqref{eq:bellman1} and~\eqref{eq:bellman2} suggests $s_t=\mu_t$, $\gamma=1$, and
\begin{align}\label{eq:reward}
r_t = \left\{\begin{array}{ll}
-\lambda_\alpha & t<\tau \\
F(\mu_t) & t=\tau
\end{array} \right.\ .
\end{align}
\begin{remark}
When formalizing the problem, we assumed discrete time steps with equal duration which led to the constant reward of $-\lambda_\alpha$ for each action prior to the stopping time. Following a similar line of argument, for real robots we can define $-\tilde\lambda(a_t)$ to account for the varying delays incurred by different actions. For instance, the delay of ``move forward" action might be different from the one for ``turn left/right" actions.
\end{remark}
Therefore, the CPP problem is amenable to RL and we can apply standard RL techniques to solve the problem computationally by deploying the robot for covering the environment and collecting its data for learning. In this paper, we use deep q-network (DQN)~\cite{mnih2015human} which combines Q-learning~\cite{watkins1992q} with deep neural networks as function approximators (for more details refer to~\cite{mnih2015human}). DQN uses experience replay~\cite{thrun1993issues}, meaning that observed transition tuples $\langle s_t,a_t,r_t,s_{t+1} \rangle$ are stored in a memory  buffer and later sampled uniformly to update the network. To improve the data efficiency of DQN and expedite its convergence, it is shown in~\cite{schaul2015prioritized} that, instead of uniform sampling, the more surprising transitions should be sampled more frequently, a method which is called prioritized experience replay (PER). Rainbow-DQN, which demonstrates superior performance in comparison to other DQN variants in several Atari games~\cite{hessel2018rainbow}, combines some of the best approaches to improve DQN like double Q-learning~\cite{van2016deep}, PER~\cite{schaul2015prioritized}, dueling architecture~\cite{wang2015dueling}, multi-step learning, distributional reinforcement learning~\cite{dabney2018distributional} and noisy nets~\cite{fortunato2017noisy}. In this paper, we use a double DQN with dueling architecture and PER.

\section{Practical Considerations}

In order to implement the RL algorithms and improve the convergence rate of the learning algorithms, we propose to simplify/approximate the state representation, modify the reward function through reward shaping, and use general advantage estimation technique to reduce gradient variance in gradient descent approach. In this section, we discuss these practical considerations one-by-one.

\subsection{State Representation}

For the state representation, since we do not have knowledge of the environment model, we cannot calculate the joint posterior distribution $\mu_t$. Hence, in this paper, we keep the estimates for the environment map, the location of the robot, and the coverage as the state. To this end, depending on whether the size of the environment is known or not, we consider two different cases.

\paragraph{Area of the environment is known:}
\label{subsubsec:known}

The dimensions of the state is  proportional to the area of the indoor environment. We stack $3$ matrices of size $n_1\times n_2$  into an $n_1\times n_2\times 3$ tensor. The first matrix records the cells that have been covered so far. The second matrix represents the obstacle locations detected by the camera or 3D sensors. Specifically, its entries are all zeros except in the locations of obstacles which are set to $1$. The third matrix captures the location of the robot, i.e., it is all zeros except for the cell that the robot is present, which is set to $1$. 
	
\paragraph{Area of the environment is unknown:}
\label{subsubsec:unknown}

We choose a dimension $n \times n$ which may be smaller than the area of the environment and $n$ is a hyper-parameter of the network. We stack the same $3$ matrices as before, except that when the size of detected area is greater than $n\times n$, we resize it to fit into $n\times n$ and apply interpolation (cubic spline in this case). The main difference would be the fact that the robot needs to store the actual-size states in the memory and prior to each call to the policy apply the resizing. 

\subsection{Reward Shaping}
\label{subsec:reward}

The reward function defined in~\eqref{eq:reward} guides the robot toward achieving the ultimate goal of covering the environment. However, since the main reward, which accounts for the coverage performance, is given at the stopping time, the convergence of RL algorithms will be slow. To address this issue, we use reward shaping technique~\cite{rewardshaping} and define 
\begin{align}\label{eq:reward2}
r_t = \left\{\begin{array}{ll}
-\lambda_\alpha+F(\mu_{t+1})-F(\mu_t) & t<\tau \\
F(\mu_t)-F(\mu_{t-1}) & t=\tau
\end{array} \right.\ .
\end{align}
On the other hand, since function $F$ is unknown, we use an approximation method. To this end, the reward of visiting an uncovered cell is set to $F(\mu_{t+1})-F(\mu_t)=+1$, whereas the reward of visiting a covered cell, i.e. an overlap, is $-\lambda$, where $\lambda$ controls the balance between coverage and overlap (or coverage time). In all the experiments we set $\lambda=0.5$. Each episode will end when either the predefined coverage level is reached or the number of the steps exceeds a threshold.

\subsection{Architecture of the DQN}
\label{subsec:architecture}

The deep neural network consists of two convolutional layers followed by two fully-connected layers.
The first convolutional layer has $16$ filters and the second one has $32$ filters, both of size $3\times 3$ and with a stride of $1$ followed by Leaky ReLU activation. The final hidden layer consists of $64$ rectifier units and the output layer has $|A|$ rectifier units corresponding to valid actions of the agent. The set of valid actions are set to (up, down, left, right) for the agent in a grid-world environment and (move forward, rotate left, rotate right) for the robot in AWS Robomaker Simulator.

We use the Adam Optimizer \cite{kingma2014adam} with mini-batches of size $32$ in all experiments. The behavior policy during training is $\epsilon$-greedy with $\epsilon$ annealed linearly from 1 to 0.1 over the first $10,000$ steps, and fixed at $0.1$ thereafter. Huber loss is used when training our DQN agent. The target network is updated every $8,000$ steps. The discount factor $\gamma=0.99$ and the learning rate is $0.001$. For PER, we set priority parameter $\alpha=0.6$ and anneal the importance sampling weight parameter $\beta$ linearly from $0.4$ to $1$. For off-policy A2C, we train the model every $10$ episodes for $50$ epochs. We set the generalized advantage estimation parameters $\lambda=0.95$ and $\gamma=0.99$. The actor learning rate is $0.0001$ and the critic learning rate is $0.0003$.

\subsection{Hybrid RL}

Besides the pure RL algorithm proposed in the previous section, we also introduce a hybrid RL approach. Before jumping to the details, we explain BA* approach~\cite{BA*}. In BA*, the robot starts covering the free areas of the environment by using a simple zigzag movement. When it reaches to a point where there is no free area surrounding the robot, it re-positions itself to some free area identified while performing the zigzag movement. Since there can be multiple such areas as well as multiple possible paths to those free regions, BA* deploys the A* search to find the closest one. The main issues with this approach are as follows;
\begin{enumerate}
    \item The noisy sensors make the online generated map unreliable; hence rendering the A* result sub-optimal.
    \item A* identifies the best path in a greedy approach. It only considers the nearest free region (in terms of the path length). However, there may be other regions which are farther but have more free area to cover when the robot moved there. Also, The next re-positioning after the current one can be also less costly.
\end{enumerate}

Hence, we propose Hybrid RL, an algorithm that combines simple zigzag movement with an RL agent. Specifically, RL replace A* search in the BA* algorithm. If the robot is used in the same environment multiple times, e.g., a vacuum cleaner is used in the same household its entire lifetime, the robot can learn through trial and error to identify the optimal re-positioning location as well as the path to it. We use the same reward function and state representation for the Hybrid RL algorithm. The only difference in that the robot runs the zigzag movement as long as free area is available in its surrounding and collects the reward of it. 

The main gain of the Hybrid RL algorithm comes from reducing the space of all possible policies by constraining the movement of the robot to be zigzag when faced with free regions. This, in turn, leads to improved performance in terms of coverage and overlap, as well as reducing the training time and the sample complexity of the learning process. Our experiments showcase the superiority of the Hybrid RL algorithm in both aspects.

\section{Convergence Analyses of DQN-PER}

In this section we provide a proof for the improved convergence of DQN-PER compared to DQN. To this end, we start with the update rule for DQN with experience replay:
\begin{align}\label{eq:bellman}
    \theta_{t+1} = \theta_t + \alpha \mathbb{E}_{(s,a)\sim{\rho_t}}\{(\mathcal{T}^*Q_{\theta_t}(s,a)
-Q_{\theta_t}(s,a))\nabla_{\theta_t}Q_{\theta_t}(s,a)\} \ ,
\end{align}
where ${\rho_t}$ is the distribution according to which we sample from the experience replay and $\mathcal{T}^*$ is the optimal Bellman operator defined by
\begin{align}
    \mathcal{T}^*Q_{\theta_t}(s,a) = \mathbb{E}_{s'\sim T(\cdot|s,a)}\{r(s,a)+\gamma\max_{a'}\ Q_{\theta_t}(s',a')\}\ .
\end{align}
The only distinction between DQN and DQN-PER is ${\rho_t}$. In order to analyze the convergence properties of~\eqref{eq:bellman}, we apply Taylor series expansion to $Q_{\theta_{t+1}}$, from which we obtain
\begin{align}
    Q_{\theta_{t+1}}=Q_{\theta_t}+\alpha K_{\theta_t} D_{\rho_t} (\mathcal{T}^* Q_{\theta_t}-Q_{\theta_t}) + O(\|\theta_{t+1}-\theta_t\|^2) \ ,
\end{align}
where $Q_{\theta_t}\in\mathbb{R}^{|\mathcal{S}||\mathcal{A}|}$, and $D_{\rho_t}$ and $K_{\theta_t}$ is $|\mathcal{S}||\mathcal{A}| \times |\mathcal{S}||\mathcal{A}|$ matrices where $D_{\rho_t}$ is diagonal with entries being ${\rho_t}(s,a)$ and the entries of $K_{\theta_t}$ are defined by
\begin{align}
    K_{\theta_t}(s,a,s',a') = \nabla_{\theta_t}Q_{\theta_t}(s,a)^T \nabla_{\theta_t}Q_{\theta_t}(s',a')\ .
\end{align}
By selecting $\alpha$ small enough we can neglect the second order term $O(\|\theta_{t+1}-\theta_t\|^2)$. Since we want to analyze the impact of experience replay, i.e. $D_{\rho_t}$, we drop $K_{\theta_t}$ as it captures the behavior of the function approximator, and define operator $\mathcal{U}_t$ as follows:
\begin{align}
    \mathcal{U}_tQ=Q+\alpha D_{\rho_t} (\mathcal{T}^*Q-Q)\ .
\end{align}
Now, we show that the sequence $\{\mathcal{U}_t\}$ is the contraction and analyze the impact of $D_{\rho_t}$ on its convergence.
\begin{align}
    [\mathcal{U}_t Q_t - \mathcal{U}_t Q^*](s,a) = & \ (1-\alpha\rho_t(s,a))(Q_t(s,a)-Q^*(s,a))\\
    & + \alpha\rho_t(s,a)(\mathcal{T}^* Q_{\theta_t}(s,a)-\mathcal{T}^* Q^*(s,a)) \\
    \leq\ & (1-\alpha\rho_t(s,a))\|Q_t-Q^*\|_\infty\\
    & + \alpha\gamma\rho_t(s,a)\|Q_t-Q^*\|_\infty \\
    =\ & [1-(1-\gamma)\alpha\rho_t(s,a)]\|Q_t-Q^*\|_\infty \ .
\end{align}
By taking the maximum of both sides over all state-action pairs we obtain
\begin{align}
    \|\mathcal{U}_t Q_t - \mathcal{U}_t Q^*\|_\infty \leq (1-(1-\gamma)\alpha\rho_t^{\rm max})\|Q_t-Q^*\|_\infty \ .
\end{align}
Then, if $\rho_t(s,a) > 0, \forall(s,a)$ and $\alpha\in(0,1/\rho_t^{\rm max})$ where $\rho_t^{\rm max}= \max_{(s,a)}\ \rho_t(s,a)$, all $\mathcal{U}_t$ are contraction mappings and their composition converges to their common fixed point $Q^*$.

Now, let us define $\beta_t=1-(1-\gamma)\alpha\rho_t^{\rm max}$. Then, we have
\begin{align}\label{eq:fixed_point}
    \|\mathcal{U}_t Q_t - Q^*\|_\infty \leq \prod_{s=1}^t \beta_s \|Q_0-Q^*\|_\infty\ ,
\end{align}
where the condition $\rho_t(s,a) > 0$ for all $(s,a)$ ensures that $\beta_t<1$. From~\eqref{eq:fixed_point} it is clear that the rate of convergence depends on the values of $\beta_t$, and consequently $\rho_t^{\rm max}$; the farther the value of $\rho_t^{\rm max}$ from zero, the faster the convergence of the iteration. For a DQN with uniform sampling of the experience replay, the parts of the state-action space that have been visited less often are sampled less frequently, i.e., $\min_{(s,a)} \rho_t(s,a) \ll \max_{(s,a)} \rho_t(s,a)$. However, PER makes sure that all the state-action pairs are sampled often enough irrespective of their empirical probability in the replay memory buffer, i.e., it increase $\min_{(s,a)} \rho_t(s,a)$ such that $\min_{(s,a)} \rho_t(s,a)$ and $\max_{(s,a)} \rho_t(s,a)$ are close enough. Since the convergence rate of the operator composed from the sequence $\{\mathcal{U}_t\}$ is only affected by $\min_{(s,a)} \rho_t(s,a)$, it will improve significantly as it will be shown in the Simulation Section.

\section{Experiments}
\label{sec:experiments}

In this section, we evaluate the performance of our proposed RL-based coverage algorithm in both grid-world environments and AWS Robomaker simulator with a Turtlebot3 Waffle Pi robot. We compare the performance of DQN and its variants with Full Spiral-STC~\cite{Choi2009}, Smooth Spiral-STC~\cite{lee2010complete}, BA*~\cite{BA*}, BSA~\cite{gonzalez2003bsa}, Epsilon*~\cite{SongG18}, and AD Path~\cite{chen2019adaptive} which are state-of-the-art or recently proposed techniques. Examples of the grid-world and Gazebo environments used for performance evaluation are shown in figures~\ref{fig:maze_envs} and~\ref{fig:gazebo_envs}, respectively. In all simulations, we set the goal coverage performance to $90\%$.

\begin{figure}[t]
\centering
\begin{tabular}{cccc}
\includegraphics[width=0.25\linewidth]{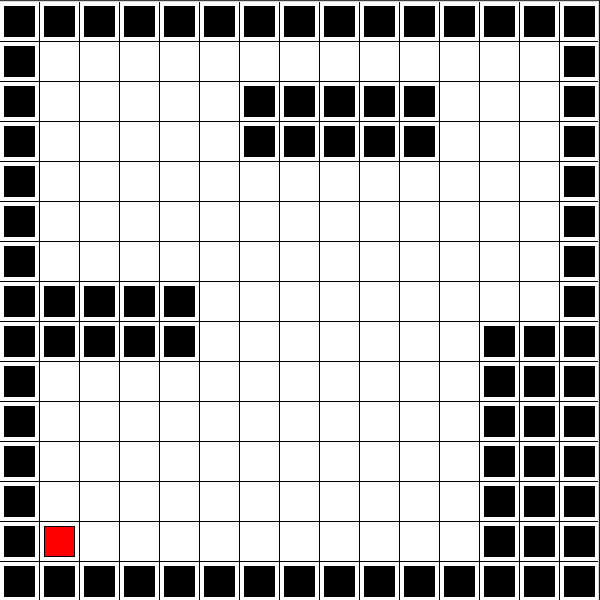} &
\includegraphics[width=0.25\linewidth]{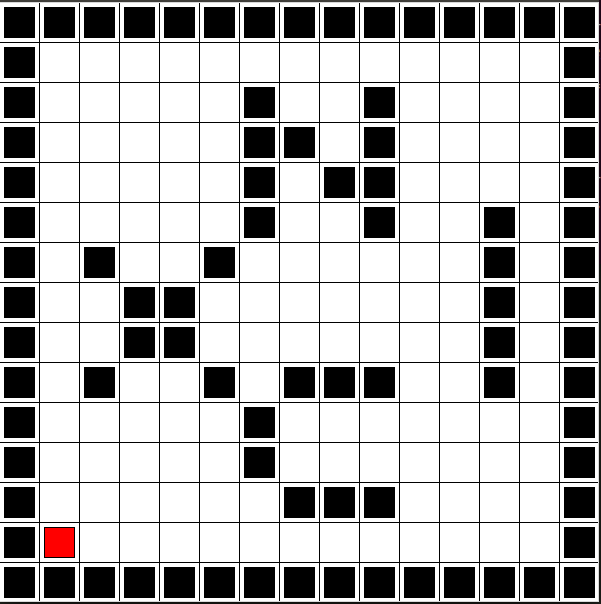} &
\includegraphics[width=0.25\linewidth]{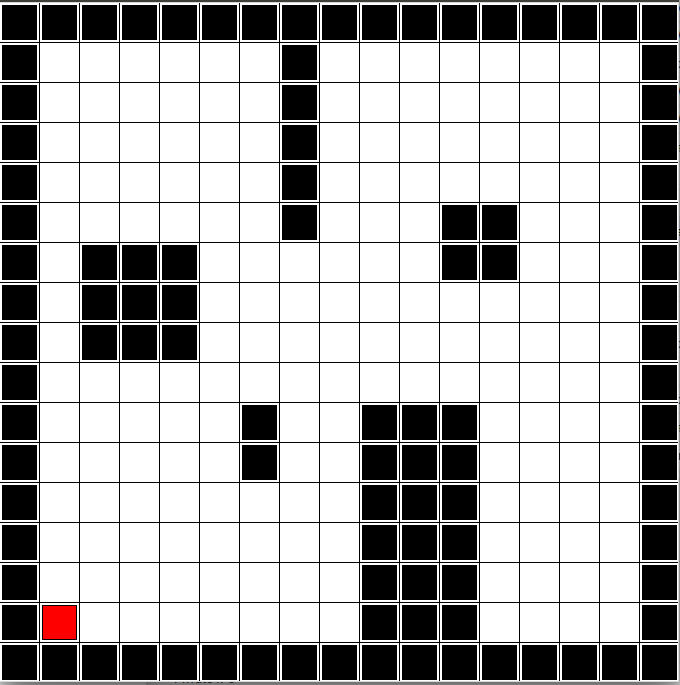} \\
(a)$15$x$15$ maze $1$& (b) $15$x$15$ maze $2$ & (c) $17$x$17$ maze $1$\\
\includegraphics[width=0.25\linewidth]{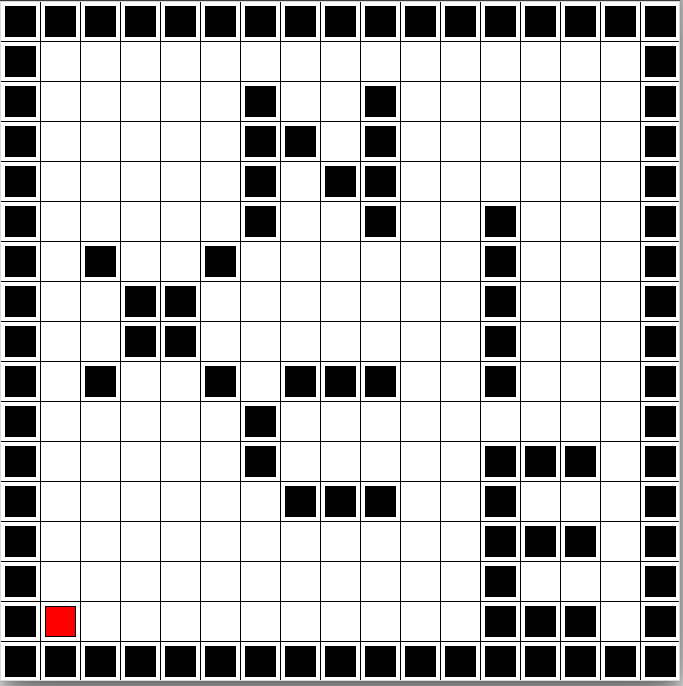} &
\includegraphics[width=0.25\linewidth]{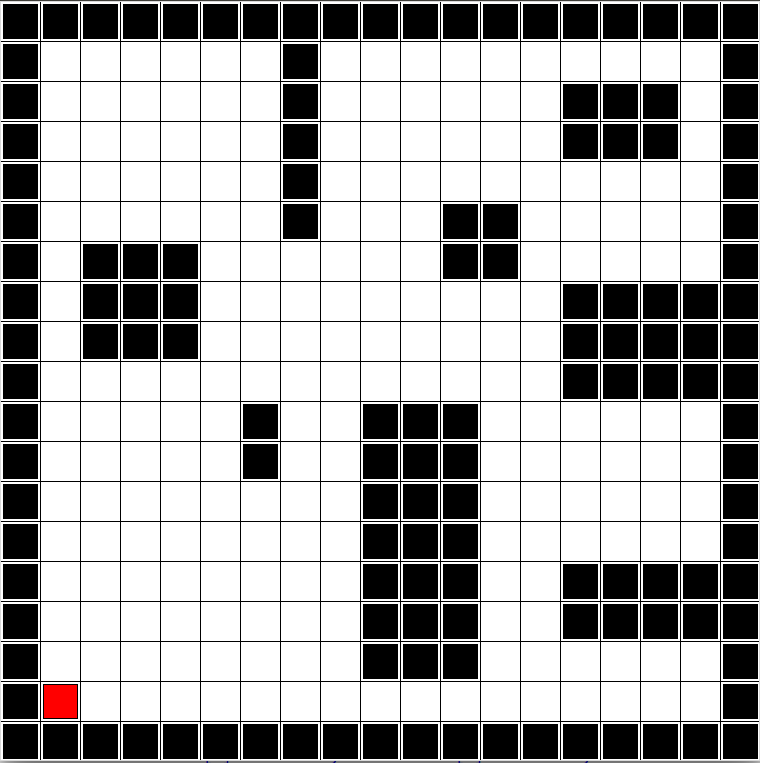} &
\includegraphics[width=0.25\linewidth]{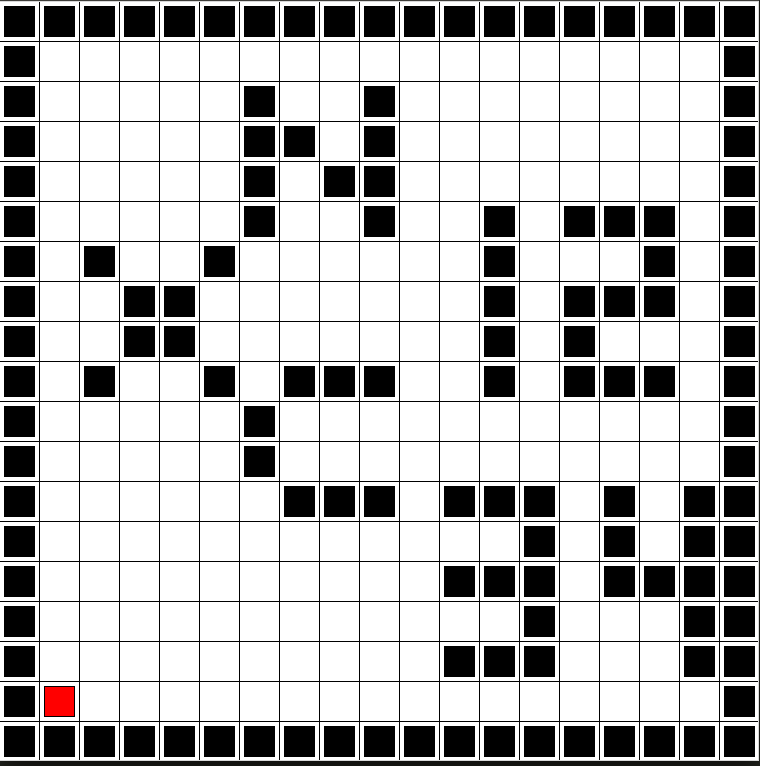} \\
(d) $17$x$17$ maze $2$& (e) $19$x$19$ maze $1$ & (f) $19$x$19$ maze $2$\\
\end{tabular}
	\vspace{-0.1in}
\caption{Example grid-world environments: white cells are free space, black cells are obstacles, and the red cell is the starting position.}
	\vspace{-0.1in}
\label{fig:maze_envs}
\end{figure}

\begin{figure}[t]
\centering
\begin{tabular}{ccc}
\includegraphics[width=0.25\linewidth]{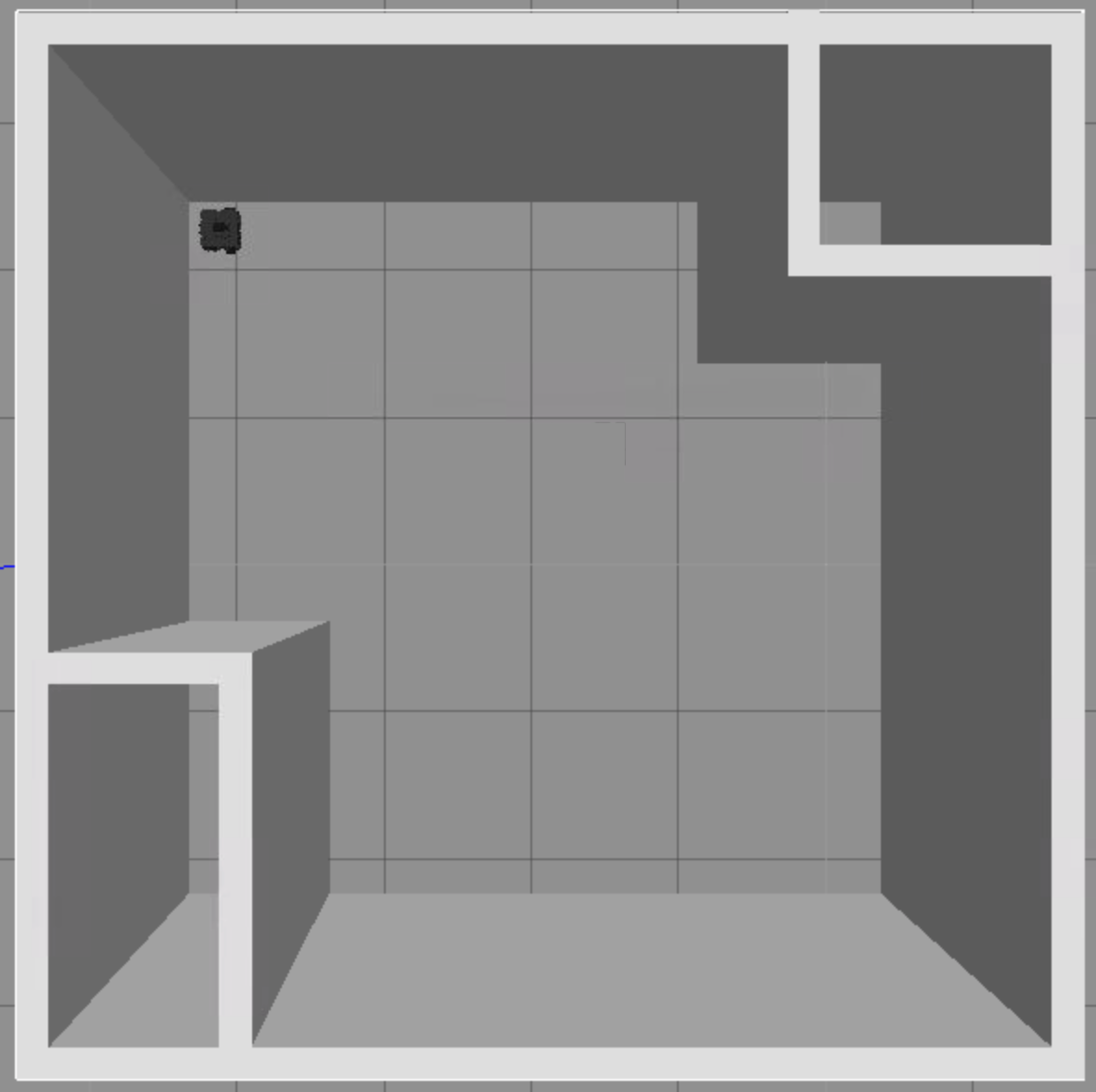} &
\includegraphics[width=0.25\linewidth]{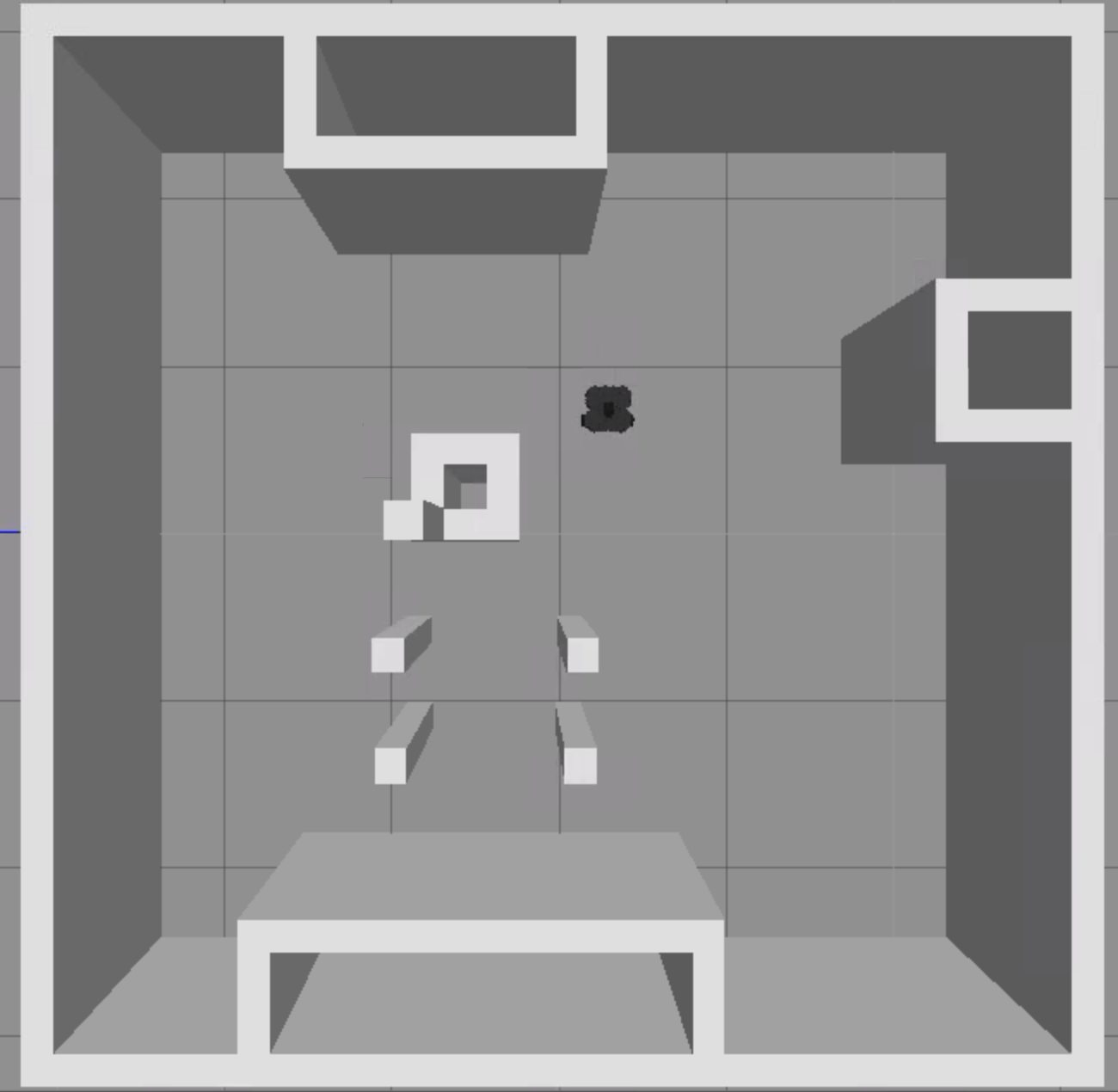} &
\includegraphics[width=0.25\linewidth]{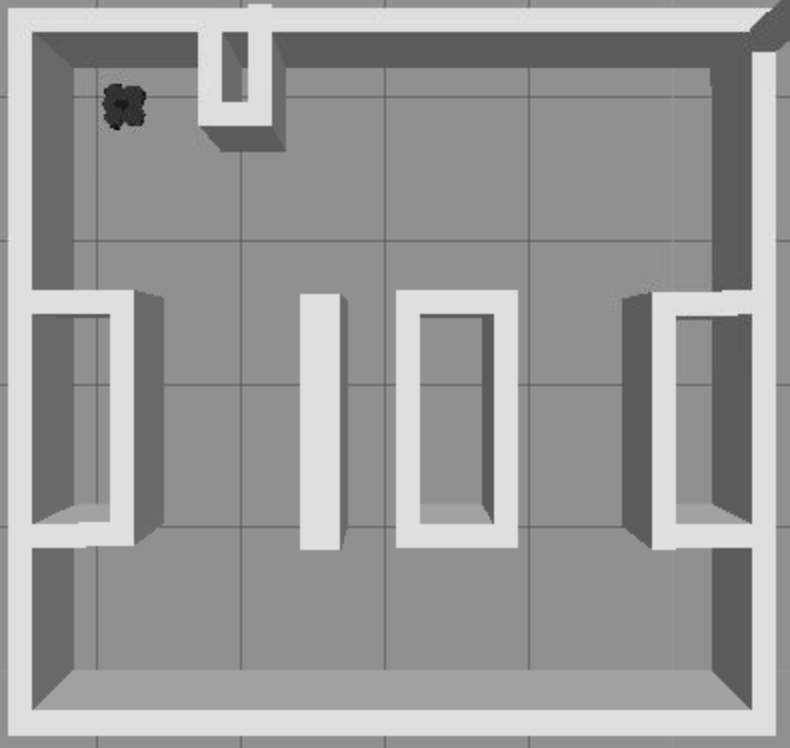}\\
(a) Env $1$& (b) Env $2$ & (c) Env $3$ \\ 
\end{tabular}
\vspace{-0.1in}
		\caption{Example indoor Gazebo environments.}
\vspace{-0.1in}
		\label{fig:gazebo_envs}
\end{figure}

\subsection{Simulations in Maze Environments}
\label{subsec:sim:maze}

\subsubsection{Comparison of Different RL Algorithms}
\label{subsubsec:rl_comp}

We evaluate the performance of different RL algorithms in the environments shown in Fig.\ref{fig:maze_envs}. Fig.~\ref{fig:known} illustrates the results in these environment. Each curve is the average of $10$ independent runs (model is randomly initialized for each run), with the variations shown in lighter colors. The performance of DQN algorithm in these environments without dynamic obstacles serve as the baseline (shown as {\fontfamily{qcr}\selectfont DQN}). It is observed that all three RL algorithms achieve the desired coverage ($90\%$), while learning to reduce overlap. Note that the performance of DQN agent does not degrade in the presence of dynamic obstacles. Moreover, the RL agent employing DQN with PER substantially reduces the sample complexity and the overlap, while indicating less variation in performance. These experiments establish the ability of RL agents to achieve the desired coverage in unknown environments with varying configurations of static and dynamic obstacles.

\begin{figure}[t]
\centering
\begin{tabular}{cc}
\includegraphics[width=0.45\textwidth]{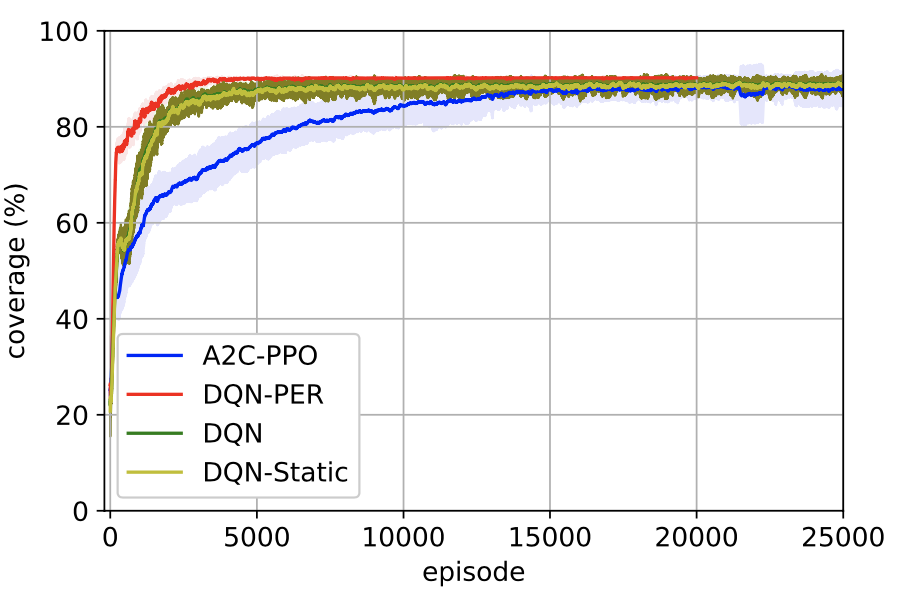} &
\includegraphics[width=0.45\textwidth]{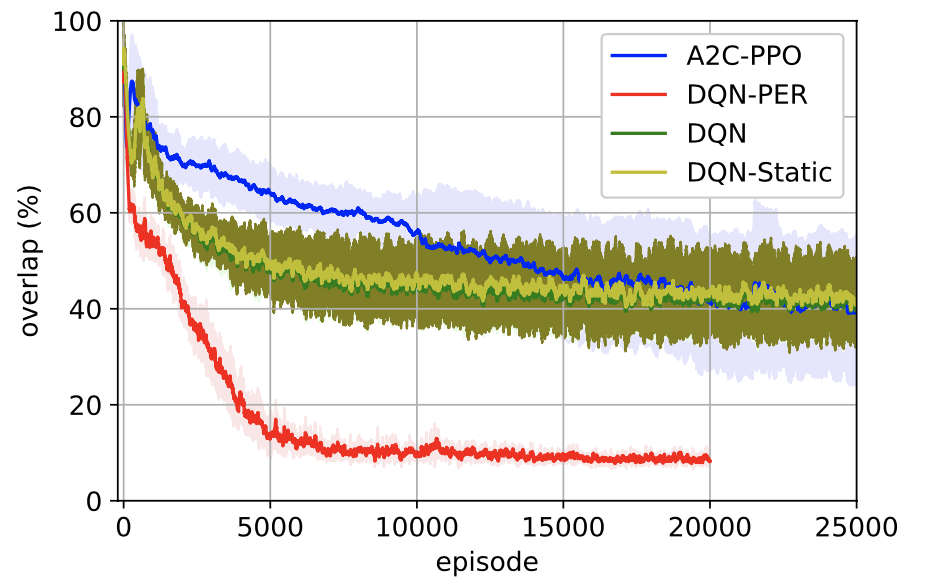}\\
(a) & (b) \\
\end{tabular}
\vspace{-0.1in}
		\caption{Average coverage (a) and overlap (b) of the agent trained on the environment in Fig.~\ref{fig:maze_envs}~(b).}
\vspace{-0.1in}
		\label{fig:known}
\end{figure}

Further, we showcase the gains we obtain by using a Hybrid RL model instead of pure RL by comparing their overlap performance on the large environment of Figure~\ref{fig:maze_envs} (f) for a constant $90\%$ coverage. It's observed in Figure~\ref{fig:hybrid} that Hybrid RL achieves the desired coverage with much lower overlap ($\approx 40\%$ compared to $\approx 12\%$), and the training process converges much faster and it's more sample efficient.

\begin{figure}[t]
\centering
\includegraphics[width=0.85\textwidth]{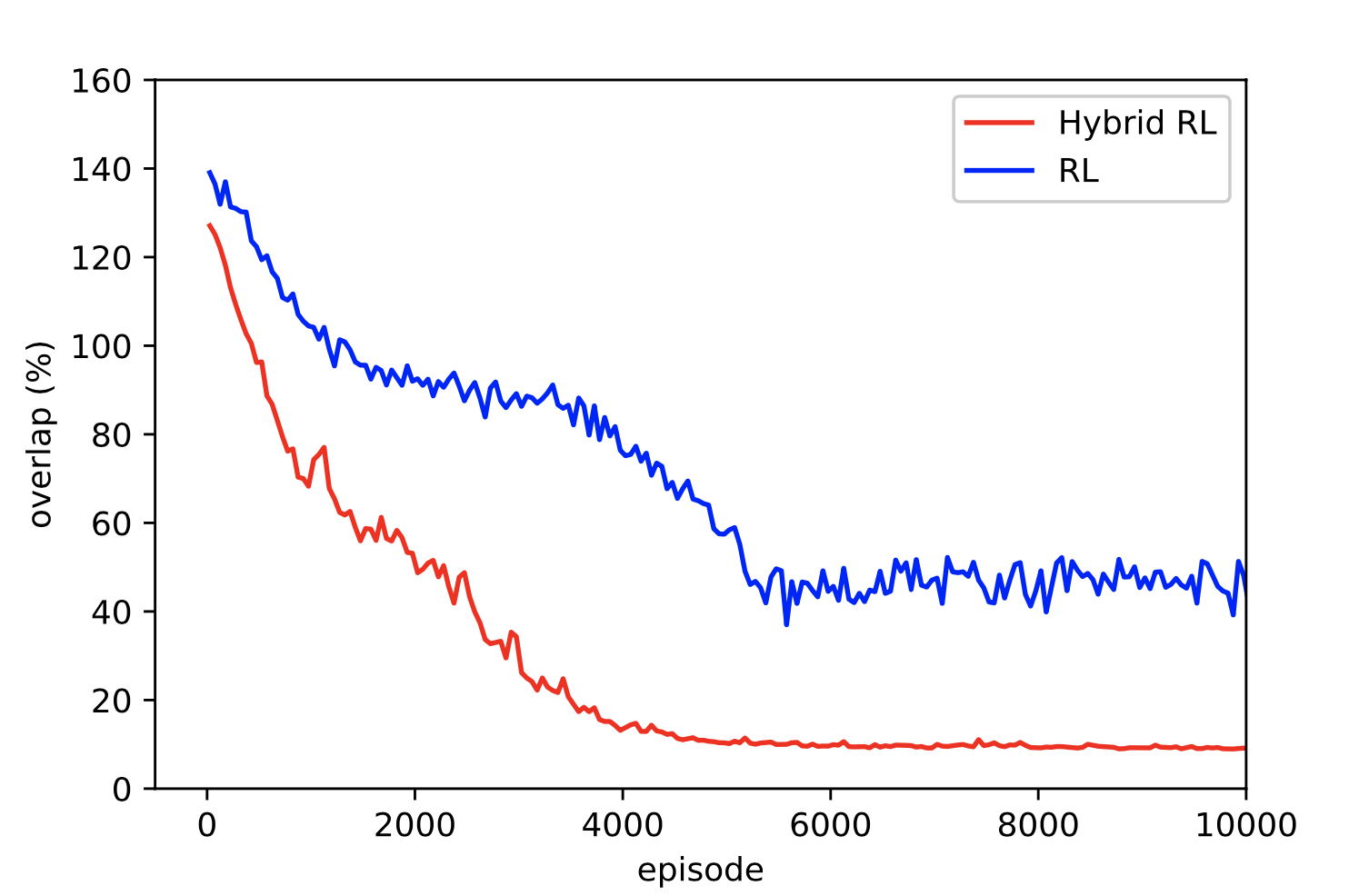} 
\caption{Average coverage (a) and overlap (b) of the agent trained on the environment in Fig.~\ref{fig:maze_envs}~(b).}
\label{fig:hybrid}
\end{figure}

\subsubsection{Comparison of RL with State of the Art Algorithms}
\label{subsubsec:_comp}

As our benchmark, we consider state of the art algorithms such as Full Spiral-STC~\cite{Choi2009}, Smooth Spiral-STC~\cite{lee2010complete}, BA*~\cite{BA*}, BSA~\cite{gonzalez2003bsa}, Epsilon*~\cite{SongG18}, and AD Path~\cite{chen2019adaptive}, where we set the desired coverage to $90\%$ and compare their overlap. Table~\ref{tab:cpp_comp} illustrates the comparison of different state of the art CPP algorithms in terms of overlap percentage. It's observed that RL and hybrid RL both outperform the current algorithms significantly in terms of overlap and hybrid approach reduces overlap of RL by approximately $20\%$.

\begin{table}
\centering
\caption{Comparing overlap ($\%$) of state of the art CPP algorithms.}
\begin{tabular}{l|cccccc}
    Method & Env 1 & Env 2 & Env 3 & Env 4 & Env 5 & Env 6 \\
    \hline
    RL & 9.1 & 9.6 & 9.7 & 10.1 & 10.0 & 10.8 \\
    Hybrid RL & $7.1$ & $7.5$ & $7.6$ & $8.0$ & $8.1$ & $8.8$ \\
    BA*\cite{BA*} & $18.3$ & $22.1$ & $21.8$ & $28.9$ & $26.3$ & $34.8$ \\
    ADP~\cite{chen2019adaptive} & $17.4$ & $20.0$ & $23.0$ & $25.5$ & $26.3$ & $36.8$ \\
    BSA~\cite{gonzalez2003bsa} & $16.4$ & $23.3$ & $22.2$ & $27.1$ & $27.4$ & $37.1$ \\
    Epsilon*~\cite{SongG18} & $23.3$ & $29.3$ & $32.1$ & $35.2$ & $34.6$ & $41.6$ \\
    Spiral STC~\cite{Choi2009} & $19.7$ & $24.2$ & $22.8$ & $29.4$ & $29.2$ & $37.4$ \\
\end{tabular}
\label{tab:cpp_comp}
\end{table}

\subsubsection{Impact of Sensor Noise}

We compare the performance of the proposed PER algorithms with the existing state of the art BA* as the benchmark~\cite{BA*}. 
Figure~\ref{fig:noise} (a) compares the overlap incurred by PER and BA* while covering $90\%$ of the $15\times15$ grid-world in Figure~\ref{fig:maze_envs} (b), which indicates the significant improvement achieved by PER.
As mentioned earlier, the main shortcoming of the existing online algorithms is that their performance degrades significantly when the sensor data are noisy. To compare the performance of PER with BA* in a noisy environment, we assume that the obstacle detection sensors have error probability $\rho$, meaning that an obstacle (free cell) is identified as a free cell (respectively obstacle) with probability $\rho$. Hence, in Figure~\ref{fig:noise} (b), we compare the overlap percentage of both methods (for PER, we use the converged value) after covering $90\%$ of the $15\times15$ grid-world in Figure~\ref{fig:maze_envs} (b) in a noisy environments. It is observed that PER is capable of learning the noise model and achieving identical performances for different noise levels while BA* performance degrades significantly.

\begin{figure}[t]
\centering
\begin{tabular}{cc}
\includegraphics[width=0.45\textwidth]{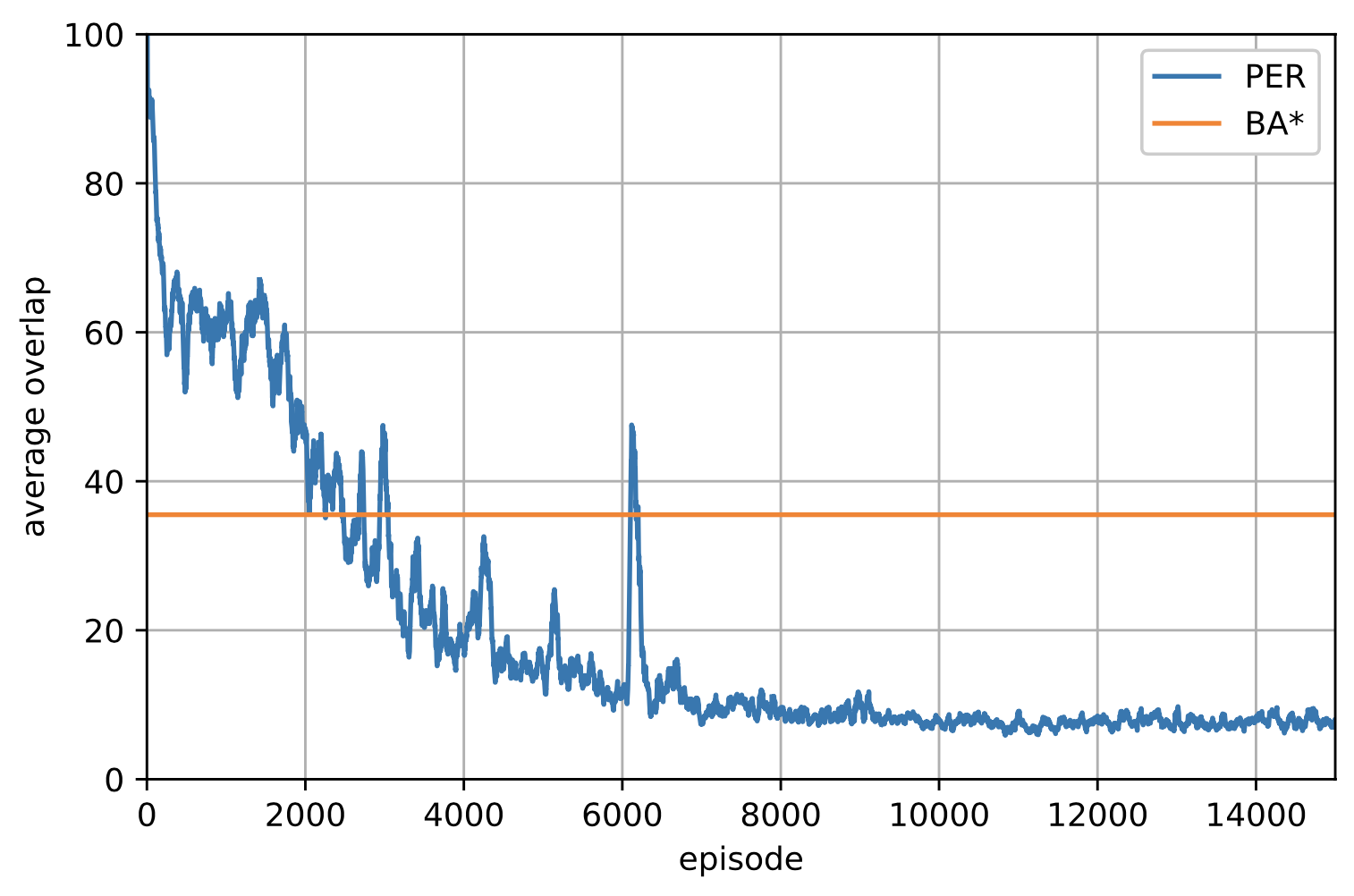} &
\includegraphics[width=0.45\textwidth]{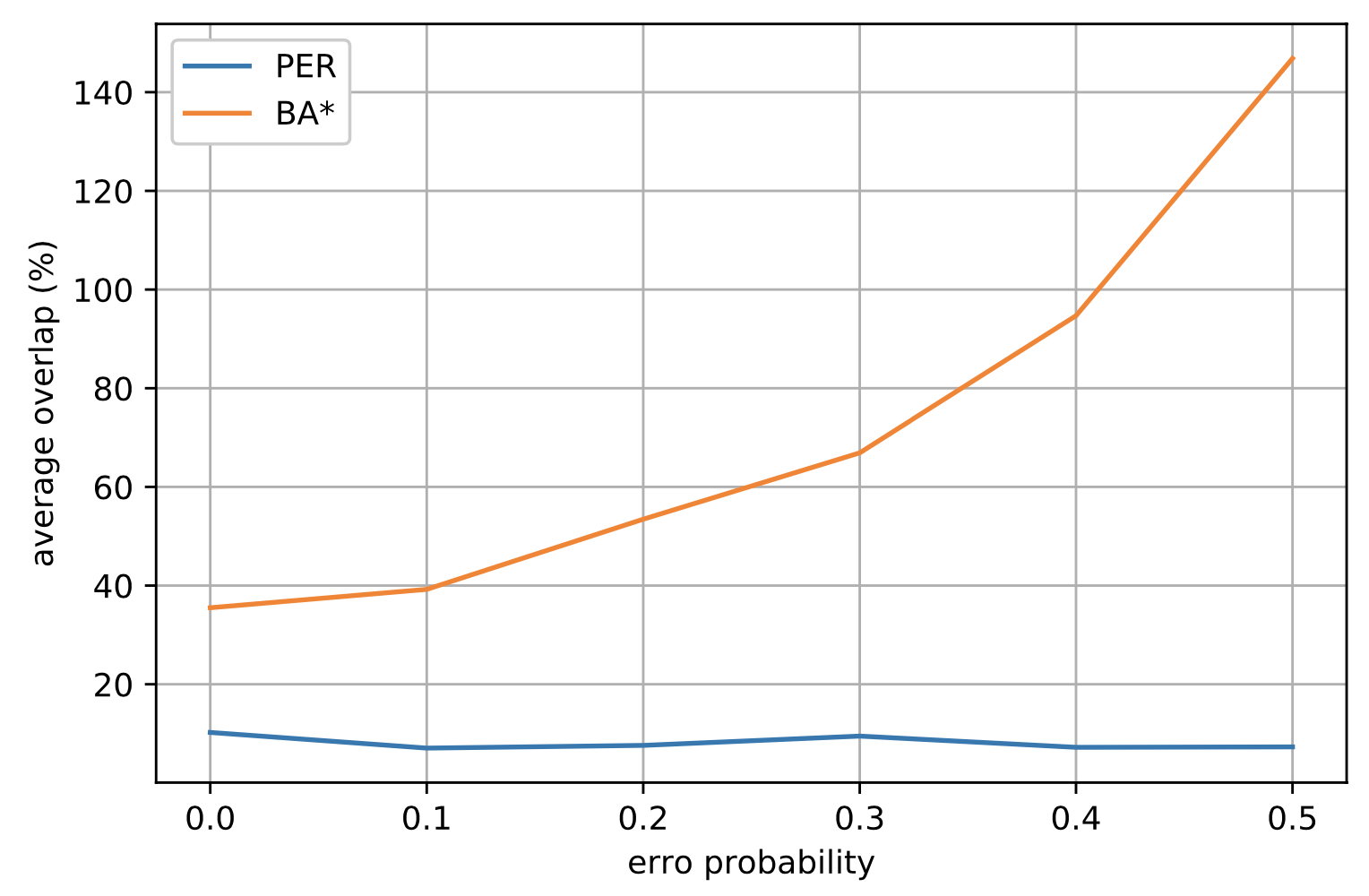}\\
(a) & (b) \\
\end{tabular}
\vspace{-0.1in}
		\caption{(a): Overlap percentage of PER and BA* in a $15\times15$ grid. (b): Comparing average overlap for different noise probabilities.}
\vspace{-0.1in}
		\label{fig:noise}
\end{figure}

\subsubsection{Sensitivity of RL Approaches to Hyper-Parameters}
\label{subsec:sensitivityhyperparameters}

We analyze the sensitivity of different RL algorithms to the hyper-parameters to ensure that for a large range of hyper-parameters the performance remains acceptable, making it more amenable for achieving coverage in real-world robots and environments. For this purpose, we identify a set of hyper-parameters, pick a range of values for each one, and compare the performance of the resulting models in terms of the average cumulative reward after convergence. Specifically, we selected batch size from $\{16,32,64,128\}$, learning rate $\{0.003, 0.001, 0.0003, 0.0001, 0.00003\}$, number of filters in the first and second convolution layers $\{[8,16], [16,32], [32, 64]\}$, the number of the fully-connected layers and the number of nodes in each layer $\{[32], [64], [128], [16,16], [16,32], [32, 32]\}$, discount factor $\{0.9, 0.95, 0.99\}$, number of epochs in PPO~\cite{schulman2017proximal} $\{20,30,40,50\}$, clipping parameter of PPO $\{0.1,0.2,0.3\}$, generalized advantage estimation parameter in PPO $\lambda\in\{0.92,0.95,0.98\}$, PER parameters $\alpha,\beta\in\{0.4,0.6\}$, and memory size $\{50000,100000,500000\}$. 

We select $10$ combinations of these parameters and provide the average reward after convergence in Fig.~\ref{fig:hp}, which demonstrates effective learning across a wide range of hyper-parameters. Specifically, we observed that network architecture and discount factor values that we selected were less decisive, while the batch size and number of epochs for PPO, $\alpha$ and $\beta$ for PER, and memory size for PER and DQN have more impact on the performance. 

\begin{figure}[t]
	\centering
	\includegraphics[width=0.9\textwidth]{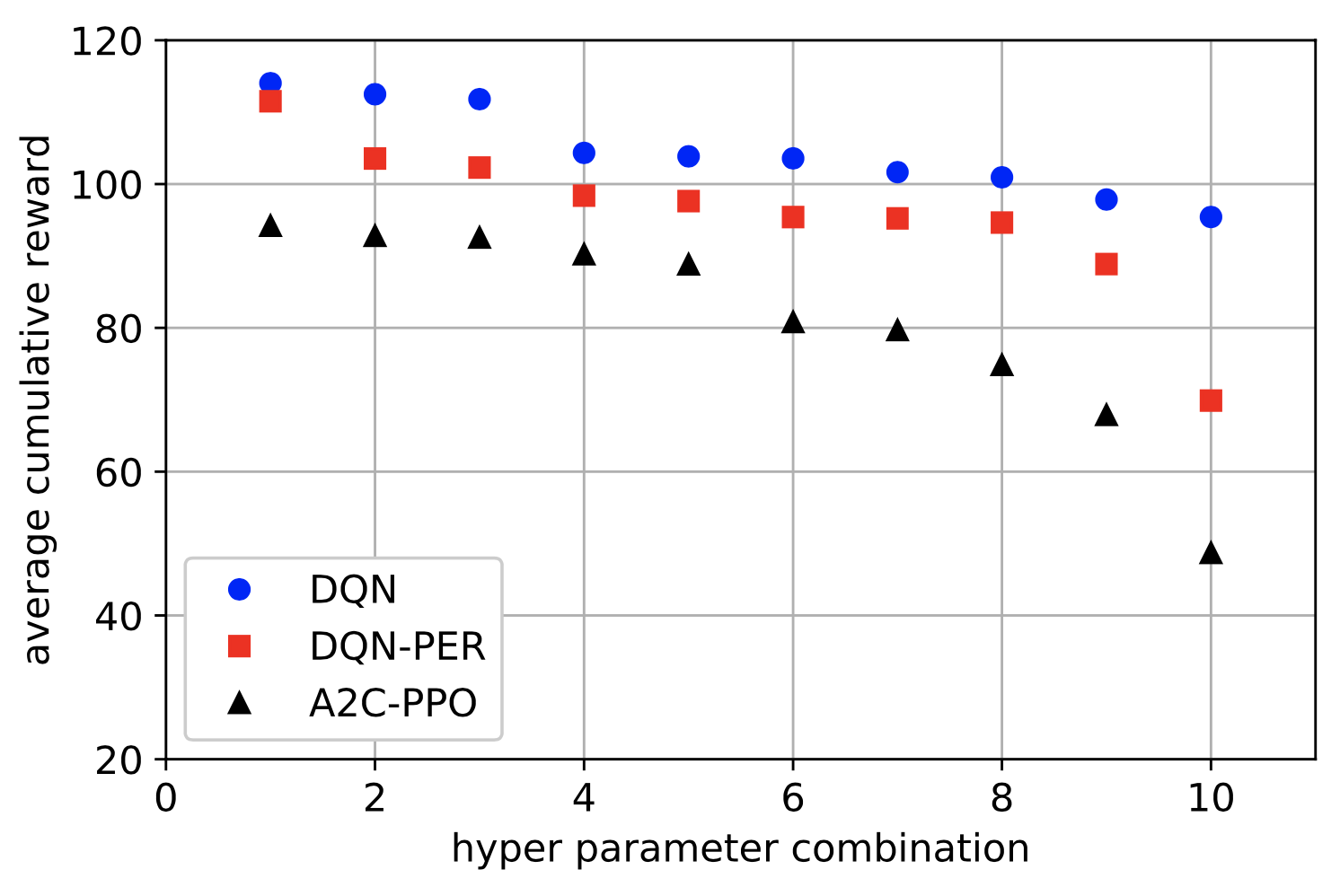} 
	\vspace{-0.1in}
	\caption{Sensitivity of area covering RL agents to hyper parameters.}
	\vspace{-0.2in}
	\label{fig:hp}
\end{figure}

\subsection{Experiments in AWS Robomaker}
\label{subsec:sim:gazebo}

In order to be able to deploy an algorithm on real-world robot, we need to take into account the physical shape of the robot and obstacles as well as the robot movement constraint and its sensing model. To this end, we use AWS Robomaker to generate typical indoor environments, and deploy Turtlebot3 Waffle Pi as the robot. We limit the robot to three actions: move forward and rotate left/right for $90$ degrees. The robot detects the presence of obstacle(s) in its vicinity through its on-board laser sensor. We perform our experiments in $3$ simulated indoor environments with a dimension of $5\times 5m^2$. meters. We have implemented our algorithm in Python programming language using the Robot Operating System (ROS), OpenAI Gym for the robot to interface with its environment(s), and AWS Robomaker for scaling up the simulations.

We compare the performance of DQN+PER (RL)~\cite{osaha2021a,osaha2021b} and Hybrid RL, with BA* in terms of coverage and overlap. During the training process for DQN, we terminate an episode when either the robot has covered $90\%$ of the free area available in the environment or has taken a predefined number of steps. Table~\ref{table3} illustrates the coverage and overlap percentage for BA*, RL and Hybrid RL, averaged over $15$ run. It shows that RL-based coverage algorithms significantly reduce overlap while achieving better or comparable coverage performance compared to BA* algorithm. Hybrid RL reduces the overlap significantly compared to RL, while we've observed $2$ orders of magnitude faster convergence, which is due to the fact that the space of possible policies reduces by constraining the movement of the robot in free area to zigzag.

\begin{table*}
    \centering
    \caption{Comparative performance of BA*, RL, and Hybrid RL.}
    \tabcolsep=0.35cm
    \begin{tabular}{|l|ll|ll|ll|}
         \hline
         \multirow{2}{*}{Method} & \multicolumn{2}{|c|}{Env 1} & \multicolumn{2}{|c|}{Env 2} & \multicolumn{2}{|c|}{Env 3} \\
         \cline{2-7}
         & Cov. & Overlap & Cov. & Overlap & Cov. & Overlap \\
         \hline
         BA* & $90.59$ & $80.61$ & $79.64$ & $87.30$ & $81.44$ & $91.70$ \\
         RL & $86.30$ & $32.69$ & $89.52$ & $17.23$ & $85.31$ & $33.25$ \\
         Hybrid & $90.59$ & $21.48$ & $90.59$ & $14.38$ & $90.59$ & $31.83$ \\
         \hline
    \end{tabular}
    \label{table3}
\end{table*}

\section{Conclusions}
In this study, the coverage path planning algorithm has been formulated as a stochastic optimization problem. After analyzing the structure and properties of the problem, it has been shown to be amenable for deep RL algorithms. Furthermore, in order to deploy RL algorithm in real-world environments, we proposed some practical methods. An efficient state representation ha been proposed to capture the history of information collected in one episode, the reward function has been derived from the structure of the problem, and the reward-shaping technique has been leveraged to expedite the convergence. Moreover, a discount factor for decreasing the gradient estimation variance has been used to further improve the convergence rate and sampling complexity. To address the scalability and universality of the RL algorithms, a hybrid RL algorithm which combines simple zigzag movement with an RL agent have been investigated, and it has been shown that RL algorithms generalize to new unseen environment better. Extensive experiments in grid-world and Gazebo environments verified the effectiveness of the RL algorithms and the proposed hybrid one.

\bibliographystyle{IEEEtran}
\bibliography{ref}

\end{document}